\documentclass{article}%
\usepackage{amsmath}
\usepackage{amsfonts}
\usepackage{amssymb}
\usepackage{graphicx}
\usepackage{algorithmic}
\usepackage{algorithm}%
\providecommand{\U}[1]{\protect \rule{.1in}{.1in}}
\newtheorem{theorem}{Theorem}
\newtheorem{corollary}{Corollary}
\newenvironment{proof}[1][Proof]{\textbf{#1.} }{\  \rule{0.5em}{0.5em}}
\begin{document}

\title{A New Computationally Simple Approach for Implementing Neural Networks with
Output Hard Constraints}
\author{Andrei V. Konstantinov and Lev V. Utkin\\Peter the Great St.Petersburg Polytechnic University\\St.Petersburg, Russia\\e-mail: andrue.konst@gmail.com, lev.utkin@gmail.com}
\date{}
\maketitle

\begin{abstract}
A new computationally simple method of imposing hard convex constraints on the
neural network output values is proposed. The key idea behind the method is to
map a vector of hidden parameters of the network to a point that is guaranteed
to be inside the feasible set defined by a set of constraints. The mapping is
implemented by the additional neural network layer with constraints for
output. The proposed method is simply extended to the case when constraints
are imposed not only on the output vectors, but also on joint constraints
depending on inputs. The projection approach to imposing constraints on
outputs can simply be implemented in the framework of the proposed method. It
is shown how to incorporate different types of constraints into the proposed
method, including linear and quadratic constraints, equality constraints, and
dynamic constraints, constraints in the form of boundaries. An important
feature of the method is its computational simplicity. Complexities of the
forward pass of the proposed neural network layer by linear and quadratic
constraints are $O(nm)$ and $O(n^{2}m)$, respectively, where $n$ is the number
of variables, $m$ is the number of constraints. Numerical experiments
illustrate the method by solving optimization and classification problems. The
code implementing the method is publicly available.


\end{abstract}

\section{Introduction}

Neural networks can be regarded as an important and effective tool for solving
various machine learning tasks. A lot of tasks require to constrain the output
of a neural network, i.e. to ensure the output of the neural network satisfies
specified constraints. Examples of tasks, which restrict the network output,
are neural optimization solvers with constraints, models generating images or
parts of images in a predefined region, neural networks solving the control
tasks with control actions in a certain interval, etc.

The most common approach to restrict the network output space is to add some
extra penalty terms to the loss function to penalize constraint violations.
This approach leads to the so-called \emph{soft} constraints or soft
boundaries. It does not guarantee that the constraints will be satisfied in
practice when a new example feeds into the neural network. This is because the
output falling outside the constraints is only penalized, but not eliminated
\cite{Marquez_Neila-etal-17}. Another approach is to modify the neural network
such that it strongly predicts within the constrained output space. In this
case, the constraints are \emph{hard} in the sense that they are satisfied for
any input example during training and inference \cite{Frerix-etal-20}.

Although many applications require the hard constraints, there are not many
models that actually realize them. Moreover, most available models are based
on applying the soft constraints due to their simple implementation by means
of penalty terms in loss functions. In particular, Lee et al.
\cite{Lee-Mehta-etal-19} present a method for neural networks that enforces
deterministic constraints on outputs, which actually cannot be viewed as hard
constraints because they are substituted into the loss function.

An approach to solving problems with conical constraints of the form $Ax\leq0$
is proposed in \cite{Frerix-etal-20}. The model generates points in a feasible
set using a predefined set of rays. A serious limitation of this method is the
need to search for the corresponding rays. If to apply this approach not only
to conical constraints, then we need to look for all vertices of the set.
However, the number of vertices may be extremely large. Moreover, authors of
\cite{Frerix-etal-20} claim that the most general setting does not allow for
efficient incorporation of domain constraints.

A general framework for solving constrained optimization problems called DC3
is described in \cite{Donti-etal-21}. It aims to incorporate (potentially
non-convex) equality and inequality constraints into the deep learning-based
optimization algorithms. The DC3 method is specifically designed for
optimization problems with hard constraints. Its performance heavily relies on
the training process and the chosen model architecture.

A scalable neural network architecture which constrains the output space is
proposed in \cite{Brosowsky-etal-21}. It is called ConstraintNet and applies
an input-dependent parametrization of the constrained output space in the
final layer. Two limitations of the method can be pointed out. First,
constraints in ConstraintNet are linear. Second, the approach also uses all
vertices of the constrained output space, whose number may be large.

A differentiable block for solving quadratic optimization problems with linear
constraints, as an element of a neural network, was proposed in
\cite{Amos-Kolter-17}. For a given optimization problem with a convex loss
function and a set of linear constraints, the optimization layer allows
finding a solution during the forward pass, and finding derivatives with
respect to parameters of the loss function and constraints during the
backpropagation. A similar approach, which embeds an optimization layer into a
neural network avoiding the need to differentiate through optimization steps,
is proposed in \cite{Agrawal-etal-19}. In contrast to \cite{Amos-Kolter-17},
the method called the Discipline Convex Programming is extended to the case of
arbitrary convex loss function including its parameters. According to the
Discipline Convex Programming \cite{Agrawal-etal-19}, a projection operator
can be implemented by using a differentiable optimization layer that
guarantees that the output of the neural network satisfies constraints.
However, the above approaches require solving convex optimization problems for
each forward pass.

Another method for solving optimization problem with linear constraints is
represented in \cite{MeiyiLi-etal-23}. It should be noted that the method may
require significant computational resources and time to solve complex
optimization problems. Moreover, it solves the optimization problems only with
linear constraints.

Several approaches for solving the constrained optimization problems have been
proposed in
\cite{Balestriero-LeCun-23,Chen-Huang-Zhang-etal-21,Detassis-etal-21,Hendriks-etal-20,Negiar-etal-23,Tejaswi-Lee-22}%
. An analysis of the approaches can be found in the survey papers
\cite{Kotary-etal-21,Kotary-etal-21a}.

To the best of our knowledge, at the moment, no approach is known that allows
building layers of neural networks, the output of which satisfies linear and
quadratic constraints, without solving the optimization problem during the
forward pass of the neural network. Therefore, we present a new
computationally simple method of the neural approximation which imposes hard
linear and quadratic constraints on the neural network output values. The key
idea behind the method is to map a vector of hidden parameters to a point that
is guaranteed to be inside the feasible set defined by a set of constraints.
The mapping is implemented by the additional neural network layer with
constraints for output. The proposed method is simply extended to the case
when constraints are imposed not only on the output vectors, but also on joint
constraints depending on inputs. Another peculiarity of the method is that the
projection approach to imposing constraints on outputs can simply be
implemented in the framework of the proposed method.

An important feature of the proposed method is its computational simplicity.
For example, the computational complexity of the forward pass of the neural
network layer implementing the method in the case of linear constraints is
$O(nm)$ and in the case of quadratic constraints is $O(n^{2}m)$, where $n$ is
the number of variables, $m$ is the number of constraints.

The proposed method can be applied to various applications. First of all, it
can be applied to solving optimization problems with arbitrary differentiable
loss functions and with linear and quadratic constraints. The method can be
applied to implement generative models with constraints. It can be used when
constraints are imposed on a predefined points or a subsets of points. There
are many other applications where the input and output of neural networks are
constrained. The proposed method allows solving the corresponding problems
incorporating the inputs as well as outputs imposed by the constraints.

Our contributions can be summarized as follows:

\begin{enumerate}
\item A new computationally simple method of the neural approximation which
imposes hard linear and quadratic constraints on the neural network output
values is proposed.

\item The implementation of the method by different types of constraints,
including linear and quadratic constraints, equality constraints, constraints
imposed on inputs and outputs are considered.

\item Different modifications of the proposed method are studied, including
the model for obtaining solutions at boundaries of a feasible set and the
projection models.

\item Numerical experiments illustrating the proposed method are provided. In
particular, the method is illustrated by considering various optimization
problems and a classification problem.
\end{enumerate}

The corresponding code implementing the proposed method is publicly available
at: 

https://github.com/andruekonst/ConstraiNet/.

The paper is organized as follows. The problem of constructing a neural
network imposing hard constraints on the network output value is stated in
Section 2. The proposed method solving the stated problem and its
modifications is considered in Section 3. Numerical experiments are given in
Section 4. Conclusion can be found in Section 5.

\section{The problem statement}

Formally, let $z\in \mathbb{R}^{d}$ denote the input data for a neural network,
and $x\in \mathbb{R}^{n}$ denote the output (prediction) of the network. The
neural network can be regarded as a function $f_{\theta}:\mathbb{R}%
^{d}\rightarrow \mathbb{R}^{n}$ such that $x=f_{\theta}(z)$, where $\theta
\in \Theta$ is a vector of trainable parameters.

Let we have a convex feasible set $\Omega \subset \mathbb{R}^{n}$ as the
intersection of a set of constraints in the form of $m$ inequalities:
\begin{equation}
\Omega=\left \{  x~|~h_{i}(x)\leq0,~i=1,...,m\right \}  ,
\label{eq:first_system}%
\end{equation}
where each constraint $h_{i}(x)\leq0$ is convex, i.e. $\forall x^{(1)}%
,x^{(2)}:h(x^{(1)}),h(x^{(2)})\leq0$, there holds
\begin{equation}
\forall \alpha \in \lbrack0,1]~(h(\alpha x^{(1)}+(1-\alpha)x^{(1)})\leq0).
\end{equation}

We aim to construct a neural network with constraints for outputs. In other
words, we aim to construct a model $x=f_{\theta}(z):\mathbb{R}^{d}%
\rightarrow \Omega$ and to impose hard constraints on $x$ such that $x\in
\Omega$ for all $z\in \mathbb{R}^{d}$, i.e.
\begin{equation}
\forall z\in \mathbb{R}^{d}~(f_{\theta}(z)\in \Omega).
\end{equation}

\section{The proposed method}

To construct the neural network with constrained output vector, two
fundamentally different strategies can be applied:

\begin{enumerate}
\item The first strategy is to project into the feasible set $\Omega$. The
strategy is to build a projective differentiable layer $P(y):\mathbb{R}%
^{n}\rightarrow \mathbb{R}^{n}$ such that $\forall y\in \mathbb{R}^{n}%
~(P(y)\in \Omega)$. A difficulty of the approach can arise with optimizing
projected points when they are outside the set $\Omega$. In this case, the
projections will lie on the boundary of the feasible set, but not inside it.
This case may complicate the optimization of the projected points.

\item The second strategy is to map the vector of \emph{hidden parameters} to
a point that is guaranteed to be inside the feasible set $\Omega$. The mapping
$G(\lambda):\mathbb{R}^{k}\rightarrow \mathbb{R}^{n}$ is constructed such that
$\forall \lambda \in \mathbb{R}^{k}~(G(\lambda)\in \Omega)$, where $\lambda$ is
the vector of hidden parameters having the dimensionality $k$. This strategy
does not have the disadvantages of the first strategy.
\end{enumerate}

In spite of the difference between the above strategies, it turns out that the
first strategy can simply be implemented by using the second strategy.
Therefore, we start with a description of the second strategy.

\subsection{The neural network layer with constraints for output}

Let a fixed point $p$ be given inside a convex set $\Omega$, i.e. $p\in \Omega
$. Then an arbitrary point $x$ from the set $\Omega$ can be represented as:
\begin{equation}
x=p+\alpha \cdot r,
\end{equation}
where $\alpha \geq0$ is a scale factor; $r\in \mathbb{R}^{n}$ is a vector (a ray
from the point $p$).\newline

On the other hand, for any $p,r$, there is an upper bound $\overline{\alpha
}_{p,r}$ for the parameter $\alpha$, which is defined as
\begin{equation}
\overline{\alpha}_{p,r}=\max \left \{  \alpha \geq0~|~p+\alpha \cdot r\in
\Omega \right \}  .
\end{equation}

At that, the segment $\left[  p;~p+\overline{\alpha}_{p,r}\cdot r\right]  $
belongs to the set $\Omega$ because $\Omega$ is convex. The meaning of the
upper bound $\overline{\alpha}_{p,r}$ is to determine the point of
intersection of the ray $r$ and one of the constraints.

Let us construct a layer of the neural network which maps the ray $r$ and the
scale factor $\alpha_{p,r}$ as follows:%
\begin{equation}
g_{p}(r,s)=p+\alpha_{p,r}(s)\cdot r, \label{eq:layer_def}%
\end{equation}
where $\alpha_{p,r}(s)$ is a function of the layer parameter $s$ and
$\overline{\alpha}_{p,r}$, which is of the form:
\begin{equation}
\alpha_{p,r}(s)=\sigma(s)\cdot \overline{\alpha}_{p,r},
\end{equation}
$\sigma(s):\mathbb{R}\rightarrow \lbrack0,1]$ is the sigmoid function, that is,
a smooth monotonic function.

Such a layer is guaranteed to fulfill the constraint
\begin{equation}
\forall r\in \mathbb{R}^{n},\ s\in \mathbb{R}~(g_{p}(r,s)\in \Omega).
\end{equation}

This neural network is guaranteed to fulfil the constraints:
\begin{equation}
\forall z\in \mathbb{R}^{d}~(f_{\theta}(z)\in \Omega),
\end{equation}
because the segment $[p,p+\overline{\alpha}_{p,r}\cdot r]$ belongs to $\Omega$.%

\begin{figure}
[ptb]
\begin{center}
\includegraphics[
height=1.985in,
width=2.6375in
]%
{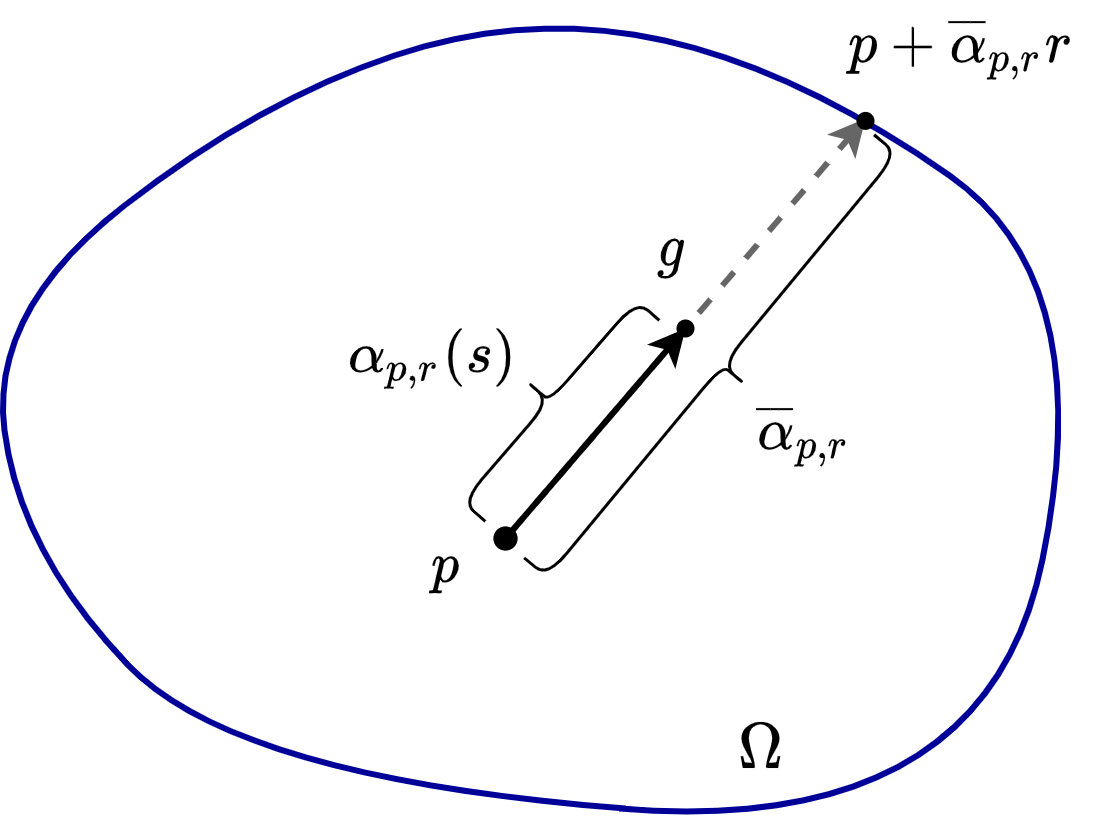}%
\caption{A scheme of the map $g_{p}(r,s)$}%
\label{fig:layer_scheme}%
\end{center}
\end{figure}

A scheme for mapping the ray $r$ and the scalar factor $s$ to a point inside
the set $\Omega$ is shown in Fig.\ref{fig:layer_scheme}. We are searching for
the intersection of the ray $r$, leaving the point $p$, with the boundary of
the set $p+\overline{\alpha}_{p,r}\cdot r$, and then the result of scaling is
the point $g$.%

\begin{figure}
[ptb]
\begin{center}
\includegraphics[
height=2.4168in,
width=3.028in
]%
{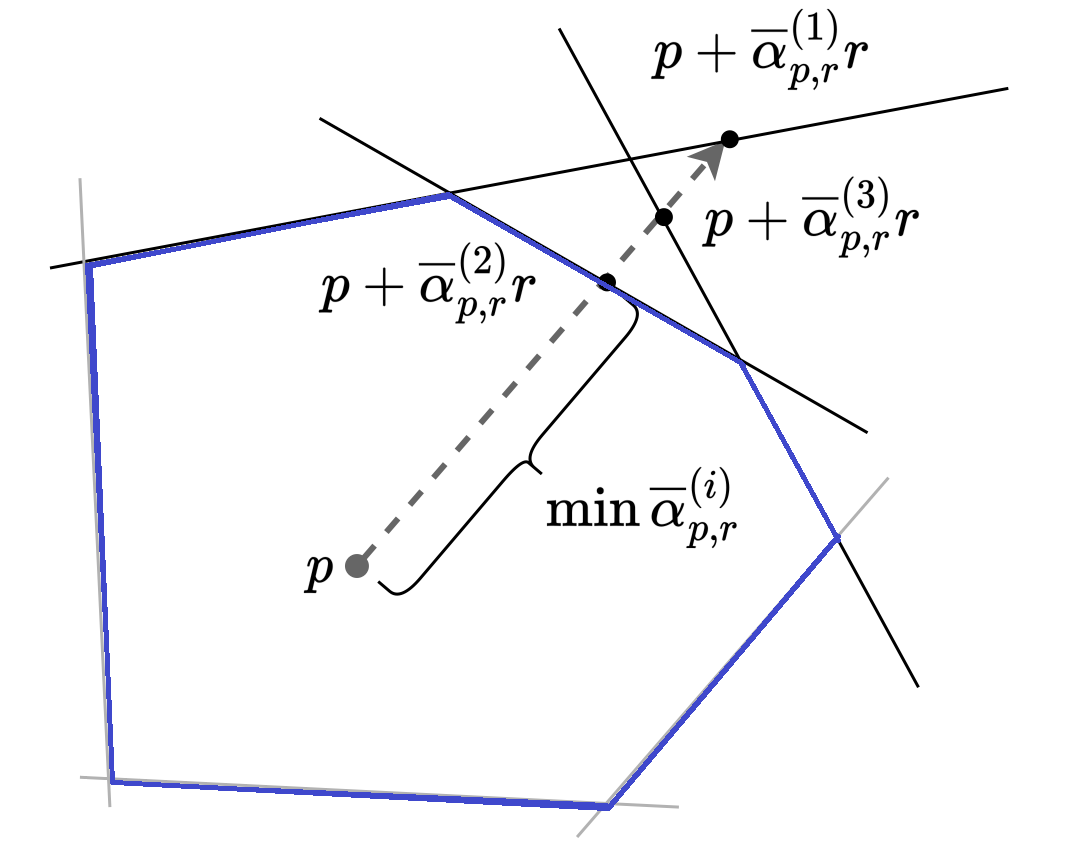}%
\caption{A scheme of searching for the upper bound $\overline{\alpha}$}%
\label{fig:linear_constraints}%
\end{center}
\end{figure}

For the entire systems of constraints, it is sufficient to find the upper
bound $\overline{\alpha}_{p,r}$ that satisfies each of the constraints. Let
$\overline{\alpha}_{p,r}^{(i)}$ be the upper bound for the parameter $\alpha$
corresponding to the $i$-th constraint $(h_{i}(x)\leq0)$ of the system
(\ref{eq:first_system}). Then the upper bound for the entire system of
constraints is determined to satisfy the condition $[p,p+\overline{\alpha
}_{p,r}\cdot r]\subseteq \lbrack p,p+\overline{\alpha}_{p,r}^{(i)}\cdot r]$,
i.e. there holds%
\begin{equation}
\overline{\alpha}_{p,r}=\min \{ \overline{\alpha}_{p,r}^{(i)}\}_{i=1}^{m}.
\end{equation}
A scheme of searching for the upper bound $\overline{\alpha}_{p,r}$, when
linear constraints are used, is depicted in Fig.\ref{fig:linear_constraints}.

Thus, the computational complexity of the forward pass of the described neural
network layer is directly proportional to the number of constraints and of the
computational complexity of intersection procedure with one constraint.

\begin{theorem}
An arbitrary vector $x\in \Omega$ can be represented by means of the layer
$g_{p}(r,s)$. The output of the layer $g_{p}(r,s)$ belongs to the set $\Omega$
for its arbitrary input $(r,s)$.
\end{theorem}

\begin{proof}
\begin{enumerate}
\item An arbitrary output vector\textbf{ }$g_{p}(r,s)$ satisfies constraints
that is $\forall r\in \mathbb{R}^{n},s\in \mathbb{R}$ there holds $g_{p}%
(r,s)\in \Omega$ because $\alpha_{p,r}(s)\leq \overline{\alpha}_{p,r}^{(i)}$,
and an arbitrary segment $[p,p+\overline{\alpha}_{p,r}^{(i)}\cdot
r]\subset \Omega$. Consequently, there holds $g_{p}(r,s)\in \lbrack
p,p+\alpha_{p,r}(s)]\subset \Omega$.

\item An arbitrary point $x\in \Omega$ can be represented by using the layer
$g_{p}(r,s)$. Indeed, let $r=x-p$, $s\rightarrow+\infty$. Then we can write
$x=g_{p}(x-p,+\infty)=p+1\cdot(x-p)=x$, as was to be proved.
\end{enumerate}
\end{proof}

\begin{corollary}
For rays from the unit sphere $||r||_{2}=1$, an arbitrary point $x\in \Omega
$,$\ x\neq p$ can be \emph{uniquely} represented by using $g_{p}(r,s)$.
\end{corollary}

In order to obtain the model $f_{\theta}(z):\mathbb{R}^{d}\rightarrow \Omega$,
outputs $r_{\theta}(z)$ and $s_{\theta}(z)$ of the neural network layers
should be fed as the input to the layer $g_{p}(r,s)$:
\begin{equation}
f_{\theta}(z)=g(r_{\theta}(z),s_{\theta}(z)).
\end{equation}

Such a combined model also forms a neural network that can be trained by the
error backpropagation algorithm.

\begin{corollary}
The output of the neural network $f_{\theta}(z)$ always satisfies the
constraints which define the set $\Omega$.
\end{corollary}

\subsection{Linear constraints}

In the case of linear constraints $\overline{\alpha}(p, r)$, the upper bound
is determined by the intersection of the ray from point $p$ to direction $r$
with the set of constraints.

Let us consider the intersection with one linear constraint of the form:
\begin{equation}
a_{i}^{T}x\leq b_{i}.
\end{equation}
Then the upper bound for the parameter $\alpha$ is determined by solving the
following system of equations:
\begin{equation}%
\begin{cases}
x=p+\alpha \cdot r,\\
a_{i}^{T}x=b_{i},\\
\alpha \geq0.
\end{cases}
\label{eq:upper_bound_system}%
\end{equation}
This implies that there holds when a solution exists:
\begin{equation}
a_{i}^{T}p+\alpha \cdot a_{i}^{T}r=b_{i},
\end{equation}%
\begin{equation}
\overline{\alpha}_{p,r}^{(i)}=\frac{b_{i}-a_{i}^{T}p}{a_{i}^{T}r}.
\end{equation}
If $a_{i}^{T}r=0$ or $\overline{\alpha}_{i}(p,r)<0$, then the system
(\ref{eq:upper_bound_system}) does not have any solution. In this case,
$\overline{\alpha}_{i}$ can be taken as $+\infty$.

Let we have the system
\begin{equation}%
\begin{cases}
a_{1}^{T}x\leq b_{1},\\
\dots \\
a_{N}^{T}x\leq b_{m},
\end{cases}
\label{eq:first_system_expanded}%
\end{equation}
and for each inequality, the upper bound $\overline{\alpha}_{p,r}^{(i)}$ is
available. Then the upper bound for the whole system of inequalities
(\ref{eq:first_system_expanded}) is determined as:
\begin{equation}
\overline{\alpha}_{p,r}=\min \left \{  \overline{\alpha}_{p,r}^{(i)}\right \}
_{i=1}^{m}.
\end{equation}

In the case of linear constraints, the computational complexity of the forward
pass of the neural network layer is $O(nm)$.

\subsection{Quadratic constraints}

Let the $i$-th quadratic constraint be given in the form:
\begin{equation}
\frac{1}{2}x^{T}P^{(i)}x+q_{i}^{T}x\leq b_{i},
\end{equation}
where the matrix $P^{(i)}$ is positive semidefinite. Then the intersection of
the ray with the constraint is given by the equation:
\begin{equation}
\frac{1}{2}(p+\alpha \cdot r)^{T}P^{(i)}(p+\alpha \cdot r)+q_{i}^{T}%
(p+\alpha \cdot r)=b_{i}.
\end{equation}

It is equivalent to the equation:
\begin{equation}
(r^{T}P^{(i)}r)\alpha^{2}+2(p^{T}P^{(i)}r+q_{i}^{T}r)\alpha+(2q_{i}^{T}%
p+p^{T}P^{(i)}p-2b_{i})=0.
\end{equation}

Depending on the coefficient at $\alpha^{2}$, two cases can be considered:

\begin{enumerate}
\item If $r^{T}P^{(i)}r=0$, then the equation is linear and has the following
solution:
\begin{equation}
\alpha=-\frac{q_{i}^{T}p+\frac{1}{2}p^{T}P^{(i)}p-b_{i}}{p^{T}P^{(i)}%
r+q_{i}^{T}r}.
\end{equation}

\item If $r^{T}P^{(i)}r>0$, then there exist two solutions. However, we can
select only the larger positive solution corresponding to the movement in the
direction of the ray. This solution is:
\begin{equation}
\alpha=-\frac{-(p^{T}P^{(i)}r+q_{i}^{T}r)+\sqrt{D/4}}{r^{T}P^{(i)}r},
\end{equation}%
\begin{equation}
\frac{D}{4}=(p^{T}P^{(i)}r+q_{i}^{T}r)^{2}-(r^{T}P^{(i)}r)\cdot(2q_{i}%
^{T}p+p^{T}P^{(i)}p-2b_{i}),
\end{equation}
because the denominator is positive.
\end{enumerate}

It should be noted that the case $r^{T}P^{(i)}r<0$ is not possible because the
matrix is positive semidefinite. Otherwise, the constraint would define a
non-convex set.

If $\alpha \geq0$, then the upper bound is $\overline{\alpha}_{p,r}%
^{(i)}=\alpha$. Otherwise, if the ray does not intersects the constraint, then
there holds $\overline{\alpha}_{p,r}^{(i)}=+\infty$.

Similarly to the case of linear constraints, if a system of the following
quadratic constraints is given:
\begin{equation}%
\begin{cases}
\frac{1}{2}x^{T}P^{(1)}x+q_{1}^{T}x\leq b_{1},\\
\dots \\
\frac{1}{2}x^{T}P^{(m)}x+q_{N}^{T}x\leq b_{m},
\end{cases}
\end{equation}
then the upper bound for the system is
\begin{equation}
\overline{\alpha}_{p,r}=\min \{ \overline{\alpha}_{p,r}^{(i)}\}_{i=1}^{m}.
\end{equation}

In the case of quadratic constraints, the computational complexity of the
forward pass of the neural network layer is $O(n^{2} m)$.

\subsection{Equality constraints}

Let us consider the case, when the feasible set is defined by a system of
linear equalities and inequalities of the form:
\begin{equation}
x\in \Omega \iff%
\begin{cases}
Ax\leq b,\\
Qx=p,
\end{cases}
\label{eq:eq_ineq_system}%
\end{equation}

In this case, the problem can be reduced to (\ref{eq:first_system}) that is it
can be reduced to a system of inequalities. In order to implement that, we
find and fix a vector $u$, satisfying the system $Qu=p$. If the system does
not have solutions, then the set $\Omega$ is empty. If there exists only one
solution, then the set $\Omega$ consists of one point. Otherwise, there exist
an infinite number of solutions, and it is sufficiently to choose any of them,
for example, by solving the least squares problem:%
\begin{equation}
||Qu-p||^{2}\rightarrow \min.
\end{equation}

Then we find a matrix $R$ which is the kernel basis matrix $Q$, that is $R$
satisfies the following condition:%
\begin{equation}
\forall w~\left(  QRw=0\right)  .
\end{equation}

The matrix $R$ can be obtained by using the SVD decomposition of
$Q\in \mathbb{R}^{\mu \times n}$ as follows:%
\begin{equation}
USV=Q,
\end{equation}
where $U\in \mathbb{R}^{\mu \times \mu}$ is the complex unitary matrix,
$S\in \mathbb{R}^{\mu \times n}$ is the rectangular diagonal matrix with
non-negative real numbers on the diagonal, $V\in \mathbb{R}^{n\times n}$ is the
conjugate transpose of the complex unitary matrix (the right singular
vectors), it contains ordered non-zero diagonal elements.

Then the matrix $R$ is defined as
\begin{equation}
R=\left(  v_{1},\dots,v_{\delta}\right)  ,
\end{equation}
where $\delta$ is the number of zero diagonal elements of $S$, $v_{1}%
,\dots,v_{\delta}$ are columns of the matrix $V$.

Hence, there holds
\begin{equation}
\forall w\in \mathbb{R}^{\delta}~\left(  Q(Rw+u)=p\right)  .
\end{equation}

A new system of constraints imposed on the vector $w$ is defined as:
\begin{equation}
A(Rw+u)\leq b,
\end{equation}
or in the canonical form:
\begin{equation}
Bw\leq t, \label{eq:new_system}%
\end{equation}
where $B=AR$, $t=b-Au$.

So, $w$ is the vector of variables for the new system of inequalities
(\ref{eq:new_system}). For any vector $w$, the vector $x$ satisfying the
initial system (\ref{eq:eq_ineq_system}) can be reconstructed as $x=Rw+u$.

In sum, the resulting model will be defined as
\begin{equation}
f_{\theta}(z)=R\widetilde{f}_{\theta}(z)+u, \label{eq:x_model_from_w}%
\end{equation}
where $\widetilde{f}_{\theta}(z)$ is the model for constraints
(\ref{eq:new_system}).

Let us consider a more general case when an arbitrary convex set as the
intersection of the convex inequality constraints (\ref{eq:first_system}) is
given, but an additional constraint is equality, i.e. there holds:
\begin{equation}
x\in \Omega \iff%
\begin{cases}
h_{i}(x)\leq0,\\
Qx=p.
\end{cases}
\label{eq:eq_arbitrary_ineq_system}%
\end{equation}

In this case, we can also apply the variable replacement to obtain new
(possibly non-linear) constraints of the form:
\begin{equation}
x\in \Omega \iff%
\begin{cases}
h_{i}(Rw+u)\leq0,\\
x=Rw+u,
\end{cases}
\iff%
\begin{cases}
\tilde{h}_{i}^{(R,u)}(w)\leq0,\\
x=Rw+u.
\end{cases}
\end{equation}

In sum, the model can be used for generating solutions $\widetilde{f}_{\theta
}(z)$, satisfying the non-linear constraints $\tilde{h}_{i}^{(R,u)}%
(\widetilde{f}_{\theta}(z))$, and then solutions for $x$ are obtained through
\ (\ref{eq:x_model_from_w}).

\subsection{Constraints imposed on inputs and outputs}

In practice, it may be necessary to set constraints not only on the output
vector $f_{\theta}(z)$, but also joint constraints depending on some inputs.
Suppose, an convex set of $\mu$ constraints imposed on the input $z$ and the
output $f_{\theta}(z)$ is given:%
\begin{equation}
\Lambda \subset \mathbb{R}^{k}\times \mathbb{R}^{n},
\end{equation}
that is, for any $z$, the model $f_{\theta}(z)$ has to satisfy:
\begin{equation}
y=%
\begin{bmatrix}
f_{\theta}(z)\\
z
\end{bmatrix}
\in \Lambda.
\end{equation}

Here $y$ is the concatenation of $f_{\theta}(z)$ and $z$. If the feasible set
is given as an intersection of convex constraints:
\begin{equation}
\Lambda=\left \{  y~|~\Gamma_{i}(y)\leq0,~i=1,...,m\right \}  .
\end{equation}

Then for a fixed $z$, a new system of constraints imposed only on the output
vector $f_{\theta}(z)$ can be built by means of the substitution:
\begin{equation}
G(z)=\left \{  z~|~\gamma_{i}(x;z)\leq0,~i=1,...,m\right \}
\end{equation}
where $\gamma_{i}(x;z)$ is obtained by substituting $z$ into $\Gamma_{i}$.

Here $\gamma_{i}$ depends on $z$ as fixed parameters, and only $x$ is a
variable. For example, if $\Gamma_{i}$ is a linear function, then, after
substituting parameters, the constraint $\gamma_{i}(x;z)\leq0$ will be a new
linear constraint on $x$, or it will automatically be fulfilled. If
$\Gamma_{i}$ is a quadratic function, then the constraint on $x$ is either
quadratic or linear, or automatically satisfied.

New \emph{dynamic} constraints imposed on the output and depending on the
input $z$ are%
\begin{equation}
f_{\theta}(z)\in G(z),
\end{equation}
under condition the input $z$ is from the admissible set
\begin{equation}
z\in \left \{  z~|~\exists x:\left[
\genfrac{}{}{0pt}{}{x}{z}%
\right]  \in \Lambda \right \}  .
\end{equation}

It can be seen from the above that the dynamic constraints can change when $z$
is changing,

\subsection{Projection model}

Note that using the proposed neural network layer with constraints imposed on
outputs, a \emph{projection} can be built as a model that maps points to the
set $\Omega$ and has the idempotency property, that is:%

\begin{equation}
\forall x \in \mathbb{R}^{n} ~ (f_{\theta}(f_{\theta}(x)) = f_{\theta}(x)).
\end{equation}

In other words, the model, implementing the identity substitution inside the
set $\Omega$ and mapping points, which are outside the set $\Omega$, inside
$\Omega$, can be represented as:
\begin{equation}%
\begin{cases}
\forall x\in \Omega~(f_{\theta}(x)=x),\\
\forall x\notin \Omega~(f_{\theta}(x)\in \Omega).
\end{cases}
\end{equation}
%

\begin{figure}
[ptb]
\begin{center}
\includegraphics[
height=1.8926in,
width=3.9479in
]%
{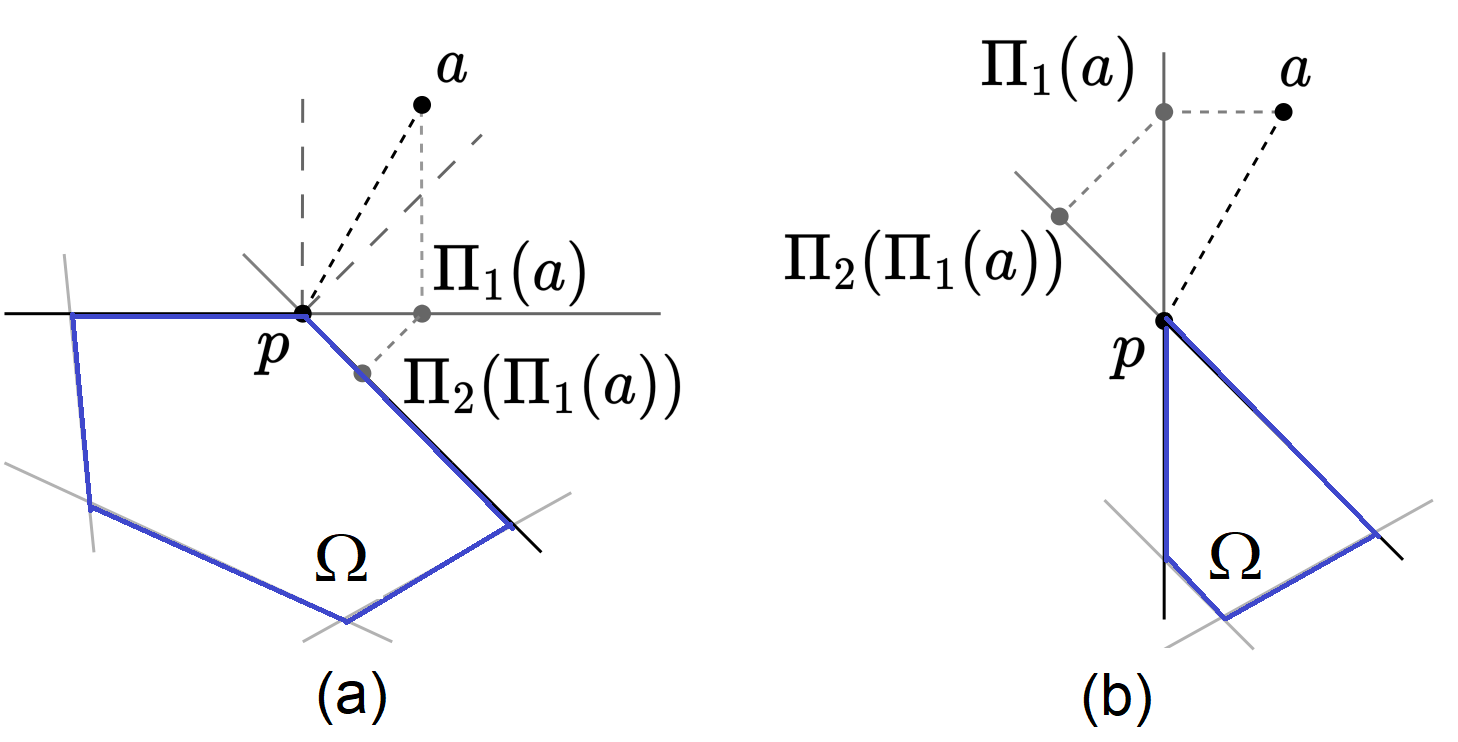}%
\caption{Difficulties of the orthogonal projection onto the intersection of
convex constraints}%
\label{fig:failed_projection}%
\end{center}
\end{figure}

The model can be implemented in two ways:

\begin{enumerate}
\item The first way is to train the model $f_{\theta}(z)=g(r_{\theta
}(z),s_{\theta}(s))$ to obtain the identity substitution by means of
minimizing the functional that penalizes the distance between the image and
the preimage. This can construct an approximation of an arbitrary projection.
For example, we can write for the orthogonal projection by using the $L_{p}%
$-norm the following functional:
\begin{equation}
\mathcal{L}=\frac{1}{N}\sum_{i=1}^{N}\left \Vert f_{\theta}(x_{i}%
)-x_{i}\right \Vert _{p}. \label{eq:l2_orthogonal_loss}%
\end{equation}

As a result, the output of the model always satisfies the constraints, but the
idempotency property cannot be guaranteed, since the minimization of the
empirical risk does not guarantee a strict equality and even an equality with
an error $\varepsilon$ on the entire set $\Omega$. Nevertheless, this approach
can be used when it is necessary to build the projective models for complex
metrics, for example, those defined by neural networks.

\item The central projection can be obtained without optimizing the model by
means of specifying the ray $r_{\theta}(z) := z - p$. In this case, the scale
factor must be specified explicitly without the sigmoid as: $\alpha_{p, r}(s)
= \min \{ 1, \overline{\alpha}_{p, r} \}$.

Then we can write%
\begin{align}
\tilde{g}_{p}(r(x))  &  =p+\min \{1,\overline{\alpha}_{p,(x-p)}\} \cdot
(x-p)\nonumber \\
&  =\left \{
\begin{array}
[c]{cc}%
x, & \overline{\alpha}_{p,(x-p)}\geq1,\\
(1-\overline{\alpha}_{p,(x-p)})p+\overline{\alpha}_{p,(x-p)}x, &
\text{otherwise.}%
\end{array}
\right.
\end{align}

It should be pointed out that other projections, for example, the orthogonal
projection by using the $L_{2}$-norm, cannot be obtained in the same way. Two
examples illustrating two cases of the relationship between $\Omega$ and a
point $a$, which has to be projected on $\Omega$, are given in
Fig.\ref{fig:failed_projection} where the orthogonal projections of the point
$a$ are denoted as $\Pi_{i}(a)$. It can be seen from
Fig.\ref{fig:failed_projection} that the point $a$ must be projected to the
point $p$ located at the intersection of constraints. The projection on the
nearest constraint as well as successive projections on constraints do not
allow mapping the point $a$ to the nearest point inside the set $\Omega$.
\end{enumerate}

%

\begin{figure}
[ptb]
\begin{center}
\includegraphics[
height=2.1837in,
width=5.5248in
]%
{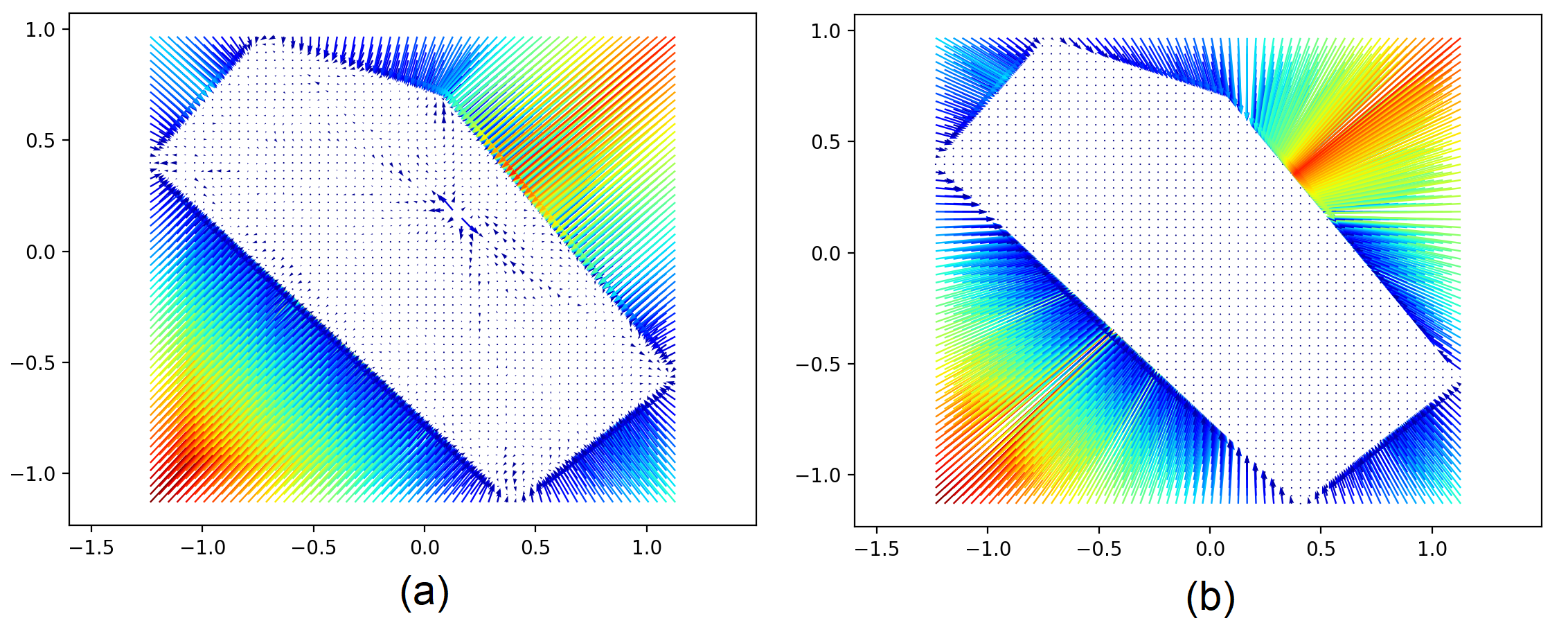}%
\caption{Illustrative examples of the projection models for linear
constraints: (a) the orthogonal approximated projection, (b) the central
projection}%
\label{fig:projection_nn_example}%
\end{center}
\end{figure}

The implementation examples of the projection model are shown in
Fig.\ref{fig:projection_nn_example}. The set $\Omega$ is formed by means of
linear constraints. For each of the examples, a vector field (the quiver plot)
depicted as the set of arrows is depicted where the beginning of each arrow
corresponds to the preimage, and the end corresponds to its projection into
the set of five constraints. On the left picture
(Fig.\ref{fig:projection_nn_example}(a)), results of the approximate
orthogonal projection implemented by a neural network consisting of five
layers are shown. The network parameters were optimized by minimizing
(\ref{eq:l2_orthogonal_loss}) with the learning rate $10^{-2}$ and the number
of iterations $1000$. It can be seen from the left picture that there are
artifacts in the set $\Omega$, which correspond to areas with the large
approximation errors. On the right picture
(Fig.\ref{fig:projection_nn_example}(b)), the result of the neural network
without trainable parameters is depicted. The neural network implements the
central projection here. It can be seen from
Fig.\ref{fig:projection_nn_example}(b) that there are no errors when the
central projection is used.%

\begin{figure}
[ptb]
\begin{center}
\includegraphics[
height=2.1494in,
width=5.6224in
]%
{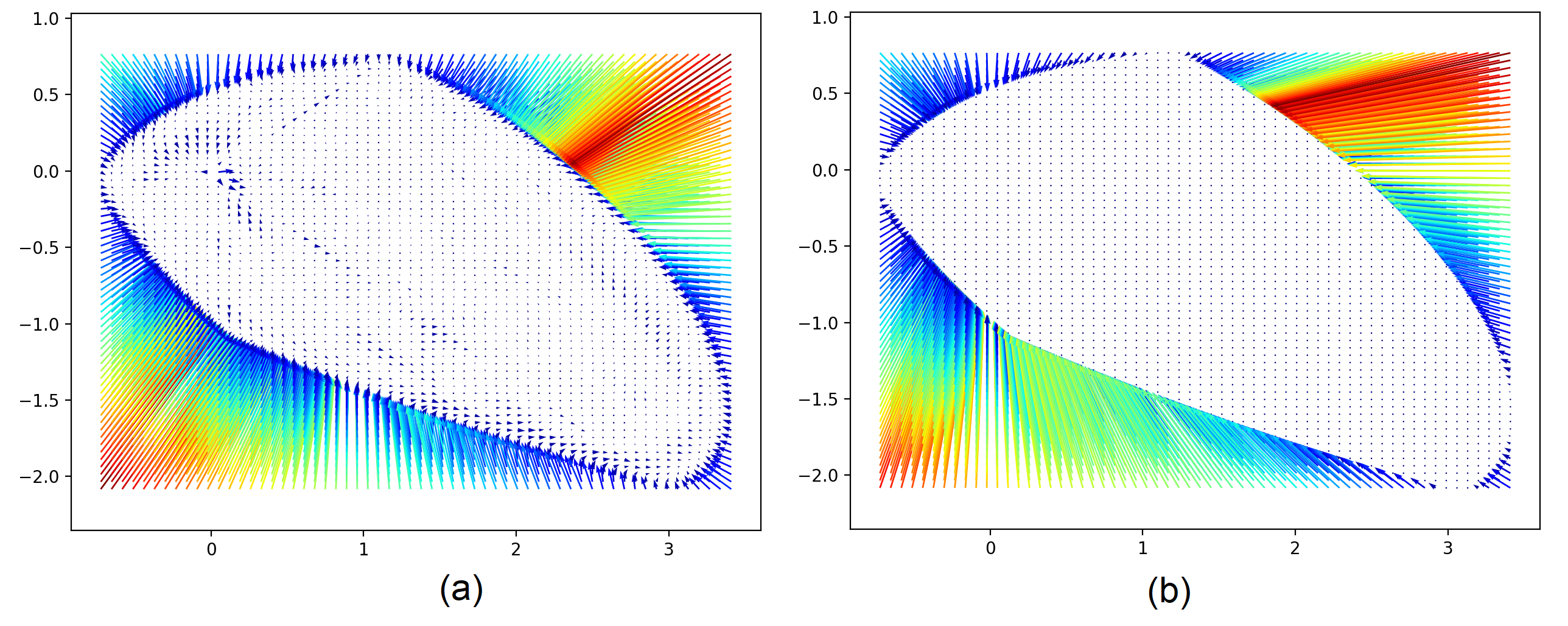}%
\caption{Examples of the projection models for quadratic constraints: (a) the
orthogonal approximated projection, (b) the central projection}%
\label{fig:quadratic_projection_nn_example}%
\end{center}
\end{figure}

Similar examples for three quadratic constraints are shown in
Fig.\ref{fig:quadratic_projection_nn_example}.

\subsection{Solutions at boundaries}

In addition to the tasks considered above, the developed method can be applied
to obtain solutions in a non-convex boundary set denoted as $\partial \Omega$.
Suppose $\sigma(s)=1$. Then we can write
\begin{equation}
g(r)=p+\overline{\alpha}(p,r)\cdot r\in \partial \Omega.
\end{equation}

The above implies that $g(r)$ is on the boundary by $r \neq \mathbf{0}$.

It is noteworthy that this approach allows us to construct a mapping onto a
non-convex connected union of convex sets. On the other hand, an arbitrary
method based on a convex combination of basis vectors, where weights of the
basis are computed using the \emph{softmax} operation, allows us to build
points only inside the feasible set, but not at the boundary.

As an example, consider the problem of projecting points onto the boundary of
a convex set:%
\begin{equation}%
\begin{array}
[c]{c}%
\min \quad||z-f_{\theta}(z)||_{p}\\
\text{s.t.}\quad f_{\theta}(z)\in \partial \Omega,
\end{array}
\label{eq:edge_proj_objective}%
\end{equation}
where $||\cdot||_{p}$ is the $p$-th norm.

Illustrative examples of projection onto an area defined by a set of linear
constraints for the $L_{1}$ and $L_{2}$-norms are shown in
Fig.\ref{fig:project_onto_edge} where the left picture
(Fig.\ref{fig:project_onto_edge}(a)) corresponds to the $L_{1}$-norm whereas
the right picture (Fig.\ref{fig:project_onto_edge}(b)) considers projections
for the $L_{2}$-norm. To solve each of the problems, a neural network
consisting of $5$ layers of size $100$ and minimizing
(\ref{eq:edge_proj_objective}) is trained. Its set of values is given as
$\partial \Omega$.%

\begin{figure}
[ptb]
\begin{center}
\includegraphics[
height=2.1318in,
width=5.5758in
]%
{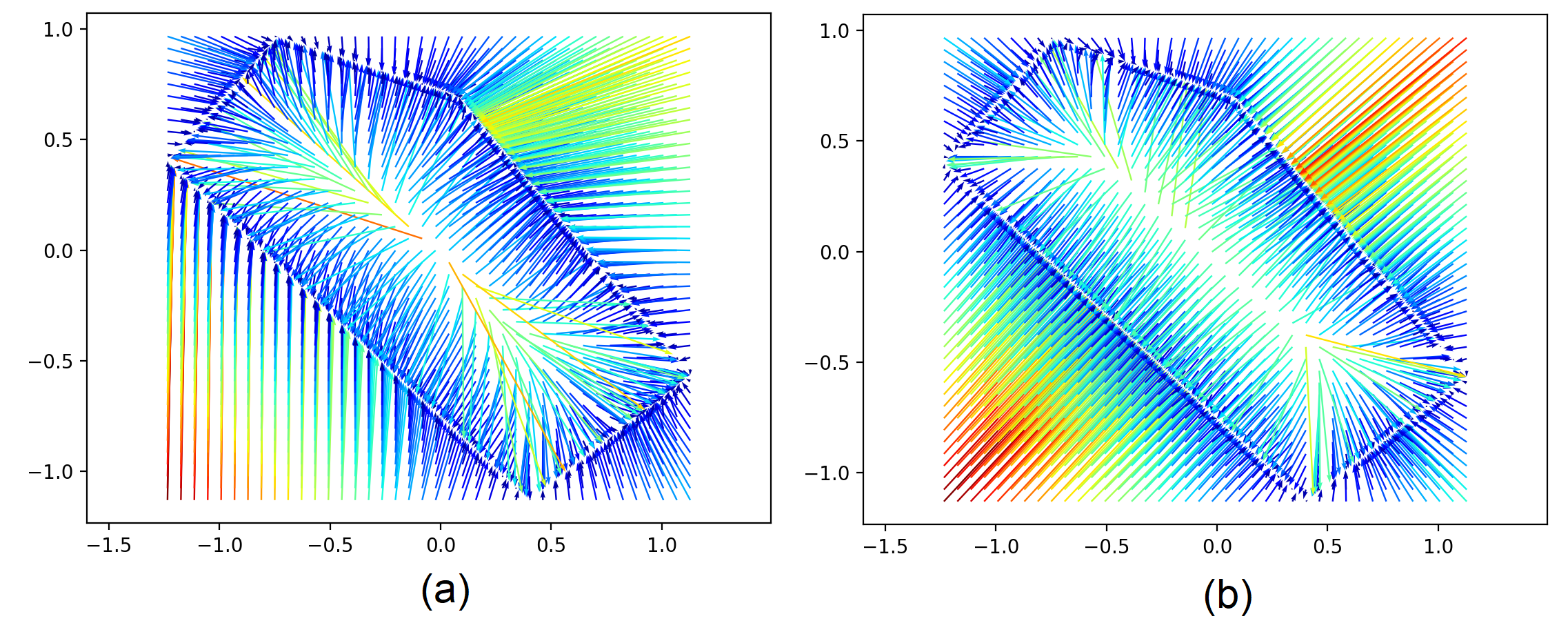}%
\caption{Illustrative examples of approximation of the projection onto the
boundary by using: (a) the $L_{1}$-norm, (b) the $L_{2}$-norm}%
\label{fig:project_onto_edge}%
\end{center}
\end{figure}

\subsection{A general algorithm of the method}

For systems of constraints containing linear equality constraints as well as
linear and quadratic inequality constraints, the general algorithm consists of
the following steps:

\begin{enumerate}
\item Eliminate the linear equality constraints.

\item Construct a new system of linear and quadratic constraints in the form
of inequalities.

\item Search for interior point $p$.

\item Train a neural network $\tilde{f}^{(\ge)}_{\theta}(z)$ for the
inequality constraints.

\item Train the final neural network $f_{\theta}(z)$, satisfying all constraints.
\end{enumerate}

\section{Numerical experiments}

\subsection{Optimization problems}

A neural network with constraints imposed on outputs should at least allow
finding a solution to the constrained optimization problems. To implement that
for each particular optimization problem, the vector of input parameters
$z_{i}$ is optimized so as to minimize the loss function $l_{i}\left(
f_{\theta}(z_{i})\right)  $. For testing the optimization method with
constraints, $50$ sets of optimization problems (objective functions and
constraints) with variables of dimensionality $2$, $5$, and $10$ are randomly
generated. In each set of problems, we generate different numbers of
constraints: $20$, $50$, $100$, and $200$. The linear and quadratic
constraints are separately generated to construct different sets of
optimization problems. The constraints are generated so that the feasible set
$\Omega$ is bounded (that is, it does not contain a ray that does not cross
boundaries of $\Omega$).

To generate each optimization problem, $m$ constraints are first generated,
then parameters of the loss functions are generated. For systems of linear
constraints, the following approach is used: a set of $m$ vectors $a_{1}%
,\dots,a_{m}\sim \mathcal{N}(0,I)$ is generated, which simultaneously specify
points belonging to hyperplanes and normals to these hyperplanes. Then the
right side of the constraint system is:
\begin{equation}
b_{i}=a_{i}^{T}a_{i},
\end{equation}
and the whole system of linear constraints:
\begin{equation}
\Omega=\left \{  x~|~a_{i}^{T}x\leq b_{i},~i=1,...,m\right \}  .
\end{equation}

For system of quadratic constraints, positive semidefinite matrices $P^{(i)}$
and vectors $q_{i}\sim \mathcal{N}(0,I)$, $i=1,...,m$, are first generated.
Then the constraints are shifted in such a way as to satisfy the constraints
with some margin $b_{i}>0$. Hence, we obtain the system of quadratic
constraints:
\begin{equation}
\Omega=\left \{  x~|~x^{T}P^{(i)}x+q_{i}^{T}x\leq b_{i},~i=1,...,m\right \}  .
\end{equation}

The relative error is used for comparison of the models. It is measured by the
value of the loss function of the obtained solution $l_{i}(f_{\theta}(x_{i}))$
with respect to the loss function of the reference solution $l_{i}(x_{i}%
^{\ast})$:
\begin{equation}
RE=\frac{\max \left \{  0,l_{i}\left(  f_{\theta}(x_{i})\right)  -l_{i}%
(x_{i}^{\ast})\right \}  }{|l_{i}(x_{i}^{\ast})|}\cdot100\%.
\end{equation}

The reference solutions are obtained using the \emph{OSQP} algorithm
\cite{Stellato-etal-20} designed exclusively for solving the linear and
quadratic optimization problems.

Tables \ref{tab:linear_loss_rel_error_percent_1} and
\ref{tab:linear_loss_rel_error_percent_2} show relative errors for
optimization problems with linear loss functions and linear and quadratic
constraints, respectively, where $m$ is the number of constraints, $n$ is the
number of variables in the optimization problems. The relative errors are
shown in tables according to percentiles ($25\%$, $50\%$, $75\%$, $100\%$) of
the probability distribution of optimization errors, which are obtained as a
result of multiple experiments. Tables
\ref{tab:quadratic_loss_rel_error_percent_1} and
\ref{tab:quadratic_loss_rel_error_percent_2} show relative errors for
optimization problems with quadratic loss functions and linear and quadratic
constraints, respectively.

It can be seen from Tables \ref{tab:linear_loss_rel_error_percent_1}%
-\ref{tab:quadratic_loss_rel_error_percent_2} that the proposed method allows
us to optimize the input parameters $r$ and $s$ for the proposed layer. One
can see that the introduced layer does not degrade the gradient for the whole
neural network.%

\begin{table}[tbp] \centering
\caption{Relative errors for problems with linear loss functions and linear constraints}%
\begin{tabular}
[c]{llrrrr}\hline
$m$ & $n$ & $25\%$ & $50\%$ & $75\%$ & $100\%$\\ \hline
& $2$ & $8.3\cdot10^{-6}$ & $3.6\cdot10^{-5}$ & $1.8\cdot10^{-4}$ &
$1.3\cdot10^{-2}$\\
$50$ & $5$ & $1.9\cdot10^{-3}$ & $2.4\cdot10^{-3}$ & $3.3\cdot10^{-3}$ &
$3.5\cdot10^{-2}$\\
& $10$ & $1.1\cdot10^{-2}$ & $1.4\cdot10^{-2}$ & $2.3\cdot10^{-2}$ &
$1.8\cdot10^{-1}$\\ \hline
& $2$ & $2.4\cdot10^{-6}$ & $2.6\cdot10^{-5}$ & $1.8\cdot10^{-4}$ &
$6.7\cdot10^{-3}$\\
$100$ & $5$ & $1.8\cdot10^{-3}$ & $2.1\cdot10^{-3}$ & $3.2\cdot10^{-3}$ &
$6.4\cdot10^{-1}$\\
& $10$ & $9.1\cdot10^{-3}$ & $1.1\cdot10^{-2}$ & $1.3\cdot10^{-2}$ &
$1.3\cdot10^{-1}$\\ \hline
& $2$ & $4.5\cdot10^{-6}$ & $1.5\cdot10^{-5}$ & $1.7\cdot10^{-4}$ &
$7.9\cdot10^{-2}$\\
$200$ & $5$ & $1.7\cdot10^{-3}$ & $2.6\cdot10^{-3}$ & $3.8\cdot10^{-3}$ &
$1.0\cdot10^{-2}$\\
& $10$ & $8.4\cdot10^{-3}$ & $1.2\cdot10^{-2}$ & $1.5\cdot10^{-2}$ &
$9.9\cdot10^{-2}$\\ \hline
\end{tabular}
\label{tab:linear_loss_rel_error_percent_1}%
\end{table}%
%

\begin{table}[tbp] \centering
\caption{Relative errors for problems with linear loss functions and quadratic constraints}%
\begin{tabular}
[c]{llrrrr}\hline
$m$ & $n$ & $25\%$ & $50\%$ & $75\%$ & $100\%$\\ \hline
& $2$ & $5.5\cdot10^{-6}$ & $1.3\cdot10^{-4}$ & $1.6\cdot10^{-3}$ &
$1.6\cdot10^{-2}$\\
$50$ & $5$ & $3.4\cdot10^{-6}$ & $7.7\cdot10^{-5}$ & $9.9\cdot10^{-4}$ &
$7.5\cdot10^{-2}$\\
& $10$ & $3.5\cdot10^{-14}$ & $4.3\cdot10^{-6}$ & $1.8\cdot10^{-4}$ &
$3.3\cdot10^{-3}$\\ \hline
& $2$ & $8.0\cdot10^{-6}$ & $1.6\cdot10^{-4}$ & $1.8\cdot10^{-3}$ &
$2.6\cdot10^{-2}$\\
$100$ & $5$ & $7.3\cdot10^{-5}$ & $8.1\cdot10^{-4}$ & $3.7\cdot10^{-3}$ &
$6.9\cdot10^{-2}$\\
& $10$ & $3.1\cdot10^{-6}$ & $1.3\cdot10^{-5}$ & $8.2\cdot10^{-4}$ &
$1.7\cdot10^{-2}$\\ \hline
& $2$ & $1.1\cdot10^{-5}$ & $3.1\cdot10^{-4}$ & $1.5\cdot10^{-3}$ &
$6.0\cdot10^{-2}$\\
$200$ & $5$ & $1.6\cdot10^{-4}$ & $1.1\cdot10^{-3}$ & $3.8\cdot10^{-3}$ &
$1.7\cdot10^{-1}$\\
& $10$ & $5.3\cdot10^{-6}$ & $2.8\cdot10^{-4}$ & $1.4\cdot10^{-3}$ &
$1.2\cdot10^{-2}$\\ \hline
\end{tabular}
\label{tab:linear_loss_rel_error_percent_2}%
\end{table}%
%

\begin{table}[tbp] \centering
\caption{Relative errors for problems with quadratic loss functions and linear constraints}%
\begin{tabular}
[c]{llrrrr}\hline
$m$ & $n$ & $25\%$ & $50\%$ & $75\%$ & $100\%$\\ \hline
& $2$ & $2.0\cdot10^{-5}$ & $3.9\cdot10^{-5}$ & $8.7\cdot10^{-5}$ &
$1.5\cdot10^{-4}$\\
$50$ & $5$ & $1.0\cdot10^{-3}$ & $1.2\cdot10^{-3}$ & $1.8\cdot10^{-3}$ &
$5.5\cdot10^{-3}$\\
& $10$ & $5.0\cdot10^{-3}$ & $6.9\cdot10^{-3}$ & $8.9\cdot10^{-3}$ &
$1.5\cdot10^{-2}$\\ \hline
$100$ & $2$ & $1.4\cdot10^{-5}$ & $5.4\cdot10^{-5}$ & $8.5\cdot10^{-5}$ &
$1.4\cdot10^{-4}$\\
& $5$ & $8.7\cdot10^{-4}$ & $1.6\cdot10^{-3}$ & $2.5\cdot10^{-3}$ &
$4.5\cdot10^{-3}$\\
& $10$ & $5.6\cdot10^{-3}$ & $8.0\cdot10^{-3}$ & $1.1\cdot10^{-2}$ &
$1.8\cdot10^{-2}$\\ \hline
$200$ & $2$ & $1.5\cdot10^{-5}$ & $5.0\cdot10^{-5}$ & $1.0\cdot10^{-4}$ &
$2.0\cdot10^{-4}$\\
& $5$ & $9.6\cdot10^{-4}$ & $1.3\cdot10^{-3}$ & $1.9\cdot10^{-3}$ &
$3.9\cdot10^{-3}$\\
& $10$ & $6.3\cdot10^{-3}$ & $7.7\cdot10^{-3}$ & $1.1\cdot10^{-2}$ &
$1.6\cdot10^{-2}$\\ \hline
\end{tabular}
\label{tab:quadratic_loss_rel_error_percent_1}%
\end{table}%
%

\begin{table}[tbp] \centering
\caption{Relative errors for problems with quadratic loss functions and quadratic constraints}%
\begin{tabular}
[c]{llrrrr}\hline
$m$ & $n$ & $25\%$ & $50\%$ & $75\%$ & $100\%$\\ \hline
& $2$ & $5.4\cdot10^{-6}$ & $7.1\cdot10^{-6}$ & $3.6\cdot10^{-5}$ &
$4.4\cdot10^{-4}$\\
$50$ & $5$ & $2.5\cdot10^{-4}$ & $4.9\cdot10^{-4}$ & $1.1\cdot10^{-3}$ &
$3.7\cdot10^{-3}$\\
& $10$ & $2.5\cdot10^{-3}$ & $3.5\cdot10^{-3}$ & $5.7\cdot10^{-3}$ &
$1.2\cdot10^{-2}$\\ \hline
& $2$ & $4.4\cdot10^{-6}$ & $6.0\cdot10^{-6}$ & $3.5\cdot10^{-5}$ &
$7.6\cdot10^{-3}$\\
$100$ & $5$ & $4.4\cdot10^{-4}$ & $8.9\cdot10^{-4}$ & $1.4\cdot10^{-3}$ &
$3.0\cdot10^{-3}$\\
& $10$ & $3.0\cdot10^{-3}$ & $4.3\cdot10^{-3}$ & $6.1\cdot10^{-3}$ &
$1.5\cdot10^{-2}$\\ \hline
& $2$ & $5.7\cdot10^{-6}$ & $1.2\cdot10^{-5}$ & $3.9\cdot10^{-5}$ &
$1.2\cdot10^{-4}$\\
$200$ & $5$ & $3.2\cdot10^{-4}$ & $5.9\cdot10^{-4}$ & $1.2\cdot10^{-3}$ &
$3.6\cdot10^{-3}$\\
& $10$ & $3.4\cdot10^{-3}$ & $4.9\cdot10^{-3}$ & $7.2\cdot10^{-3}$ &
$1.2\cdot10^{-2}$\\ \hline
\end{tabular}
\label{tab:quadratic_loss_rel_error_percent_2}%
\end{table}%

An alternative way to solve the problem is to use optimization layers proposed
in \cite{Amos-Kolter-17} or \cite{Agrawal-etal-19}. However, this solution is
not justified due to the performance reasons. For example, consider a problem
with quadratic constraints and a linear loss function from previous
experiments. In this problem, parameters, which are fed to the input of the
neural network, are optimized in such a way as to minimize the loss function
depending on the network output. By means of the optimization layer, the
problem of projection into constraints of the form
(\ref{eq:edge_proj_objective}) is solved in this case instead of solving the
original optimization problem.

To compare the method \cite{Amos-Kolter-17,Agrawal-etal-19} with the proposed
method, the library CVXPYLayers \cite{Agrawal-etal-19} is used, which allows
us to set differentiable optimization layers within the neural network.
Fig.\ref{fig:performance_comparison} compares the optimization time of the
input vector by the same hyperparameters under condition that the algorithm
CVXPYLayers is stopped after five minutes from the beginning of the
optimization process even if the optimization is not completed. It can be seen
from Fig.\ref{fig:performance_comparison} that the proposed algorithm requires
significantly smaller times to obtain the solution. It should be noted that
this experiment illustrates the inexpediency of constructing the
\emph{constrained neural networks} by solving the projection optimization
problem during the forward pass. Nevertheless, the use of such layers is
justified if, for example, it is required to obtain a strictly orthogonal projection.%

\begin{figure}
[ptb]
\begin{center}
\includegraphics[
height=2.2743in,
width=5.0121in
]%
{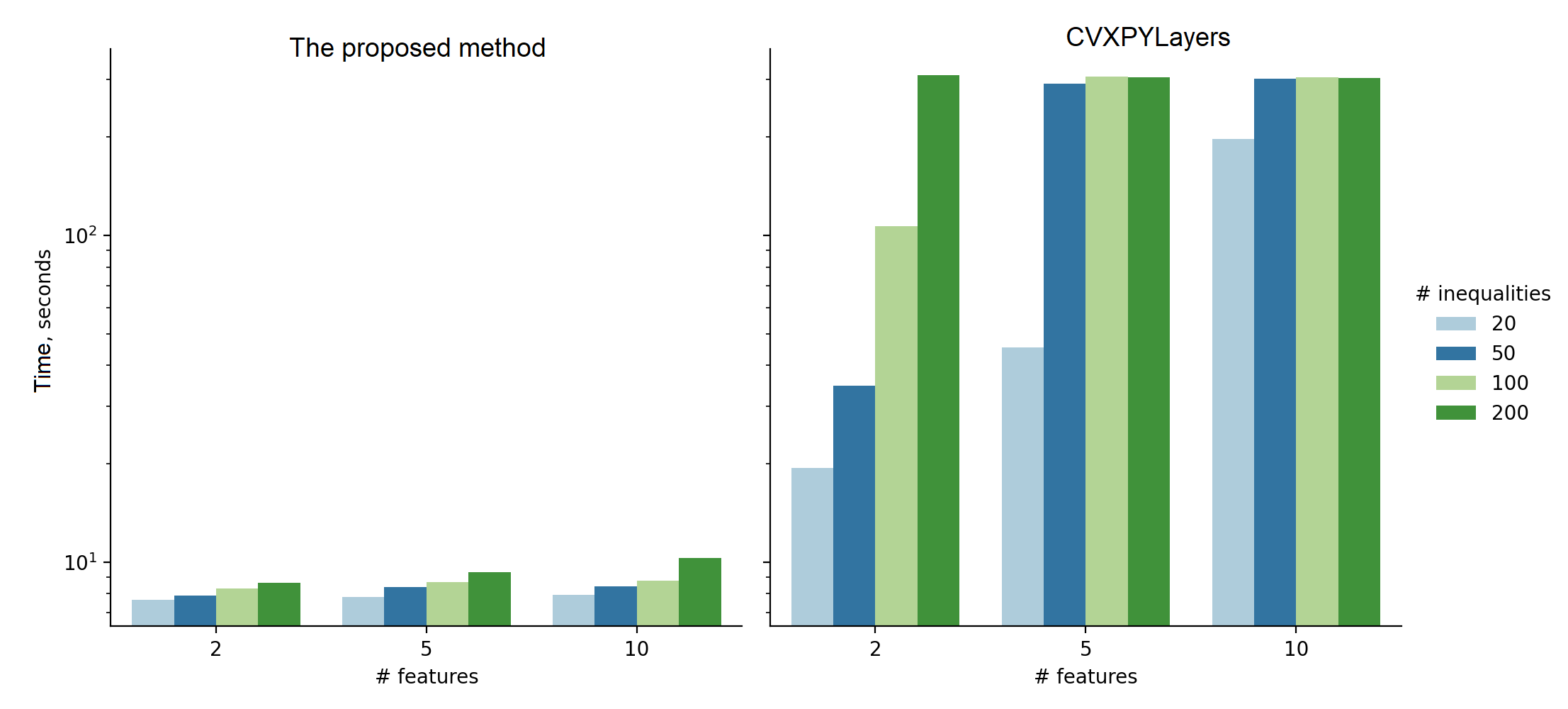}%
\caption{Comparison of the computation time of the proposed neural network and
the projection using CVXPYLayers}%
\label{fig:performance_comparison}%
\end{center}
\end{figure}

The proposed structure of the neural network allows us to implement algorithms
for solving arbitrary problems of both convex and non-convex optimization. For
example, Fig.\ref{fig:rosenbrock} shows optimization trajectories for the
Rosenbrock function \cite{Rosenbrock-1960} with quadratic constraints:
\begin{equation}
\mathcal{L}_{Ros}(x)=(1-x_{1}^{2})+100(x_{2}-x_{1}^{2})^{2},
\end{equation}%
\begin{equation}
\Omega_{Ros}=\left \{  x~|~x_{1}^{2}+x_{2}^{2}\leq2\right \}  .
\end{equation}

This function with constraints has a global minimum at the point $(1,1)$. The
bound for the constraint set $\Omega_{Ros}$ is depicted by the large black
circle. For updating the parameters, $2000$ iterations of the Adam algorithm
are used with the learning rate $0.1$. $9$ points on a uniform grid from
$-0.75$ to $0.75$ are chosen as starting points. For each starting point, an
optimization trajectory is depicted such that its finish is indicated by a
white asterisk. Three different scenarios are considered:

\begin{enumerate}
\item[(a)] \emph{The central projection} optimizes the input of a layer with
the constrained output that performs the central projection. Such a layer
implements the identity mapping inside the constraints and maps the outer
points to the boundary.

\item[(b)] \emph{The hidden space} $(r,s)$ optimizes the input parameters of
the proposed layer with constraints ($r$ is the ray, $s$ is the scalar that
defines a shift along the ray).

\item[(c)] \emph{The projection neural network} is a neural network which
consists of $5$ fully connected layers of size $100$ and the proposed layer
with constraints. The input parameters of the entire neural network are optimized.
\end{enumerate}

%

\begin{figure}
[ptb]
\begin{center}
\includegraphics[
height=1.817in,
width=5.6937in
]%
{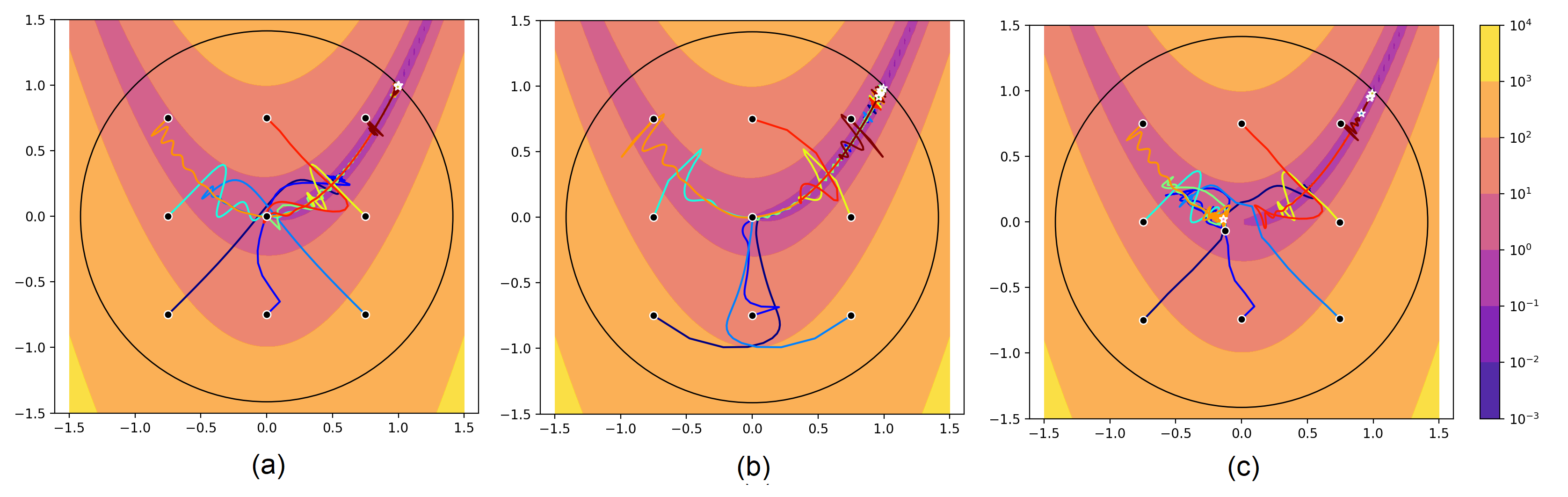}%
\caption{Optimization trajectories for the Rosenbrock function with quadratic
constraints: (a) the central projection, (b) the hidden space $(r,s)$, (c) the
projection neural network}%
\label{fig:rosenbrock}%
\end{center}
\end{figure}

Fig.\ref{fig:bird} shows the optimization trajectories for the non-convex Bird
function from \cite{Mishra-06}:
\begin{equation}
\mathcal{L}_{Bird}(x)=sin(x_{2})e^{(1-cos(x_{1}))^{2}}+cos(x_{1}%
)e^{(1-sin(x_{2}))^{2}}+(x_{1}-x_{2})^{2}, \label{eq:bird}%
\end{equation}%
\begin{equation}
\Omega_{Bird}=\left \{  x~|~(x_{1}+5)^{2}+(x_{2}+5)^{2}<25\right \}  .
\end{equation}

This function has four local minima in the region under consideration, two of
which lie on the boundary of the set $\Omega_{Bird}$, which is depicted by the
large black circle in Fig.\ref{fig:bird}.%

\begin{figure}
[ptb]
\begin{center}
\includegraphics[
height=1.8627in,
width=5.7385in
]%
{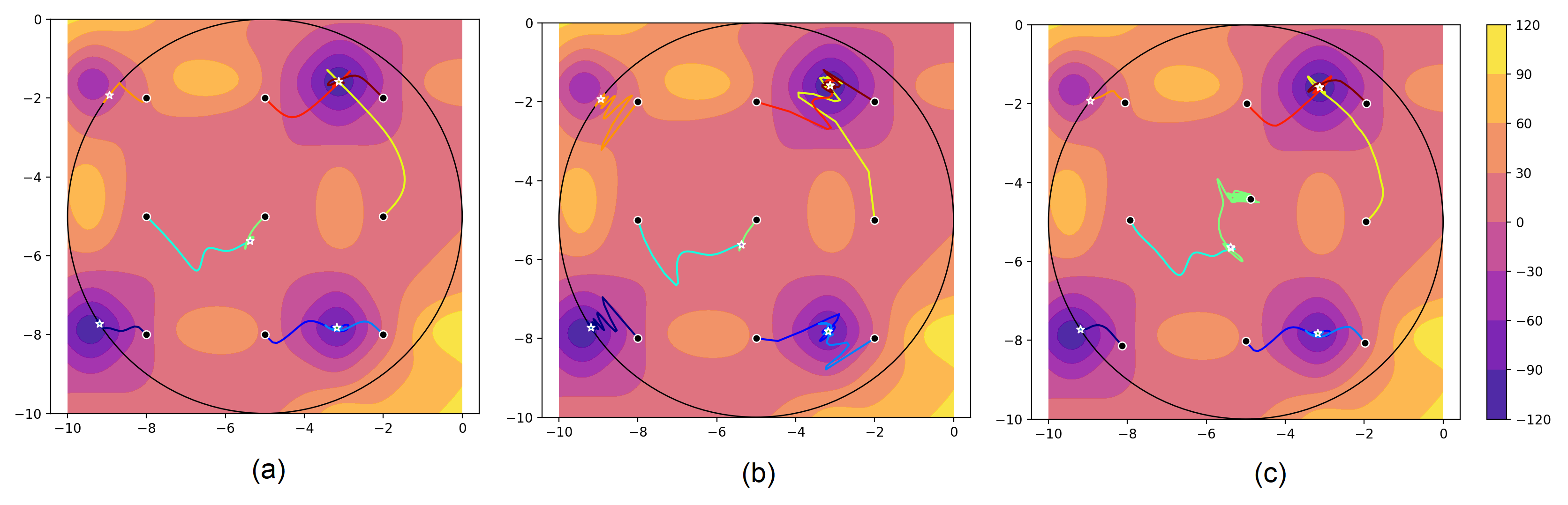}%
\caption{Optimization trajectories for the Bird function \ref{eq:bird}: (a)
the central projection, (b) the hidden space $(r,s)$, (c) the projection
neural network}%
\label{fig:bird}%
\end{center}
\end{figure}

\subsection{A classification example}

In order to illustrate capabilities of neural networks with the output
constraints, consider the classification problem by using an example with the
dataset Olivetti Faces taken from package \textquotedblleft
Scikit-Learn\textquotedblright. The dataset contains $400$ images of the size
$64\times64$ divided into $40$ classes. We construct a model whose output is a
discrete probability distribution that is
\begin{equation}
f_{\theta}(z)\in \left \{  x~|~x_{i}\geq0,~\mathbf{1}^{T}x=1\right \}  .
\label{eq:proba_distr}%
\end{equation}

It should be noted that traditionally the \emph{softmax} operation is used to
build a neural network whose output is a probability distribution.

For comparison purposes, Fig.\ref{fig:classification_comparison} shows how the
loss functions depend on the epoch number for the training and testing
samples. Each neural network model contains $5$ layers of size $300$ and is
trained using \emph{Adam} on $5000$ epochs with the batch size $200$ and the
learning rate $10^{-4}$ to minimize the cross entropy. Three types of final
layers are considered to satisfy the constraints imposed on the probability
distributions (\ref{eq:proba_distr}):

\begin{itemize}
\item \emph{Constraints} means that the proposed layer of the neural network
imposes constraints on the input $(r,s)$;

\item \emph{Projection} means that the proposed layer projects the input to
the set of constraints;

\item \emph{Softmax} is the traditional \emph{softmax} layer.
\end{itemize}

%

\begin{figure}
[ptb]
\begin{center}
\includegraphics[
height=2.488in,
width=5.5943in
]%
{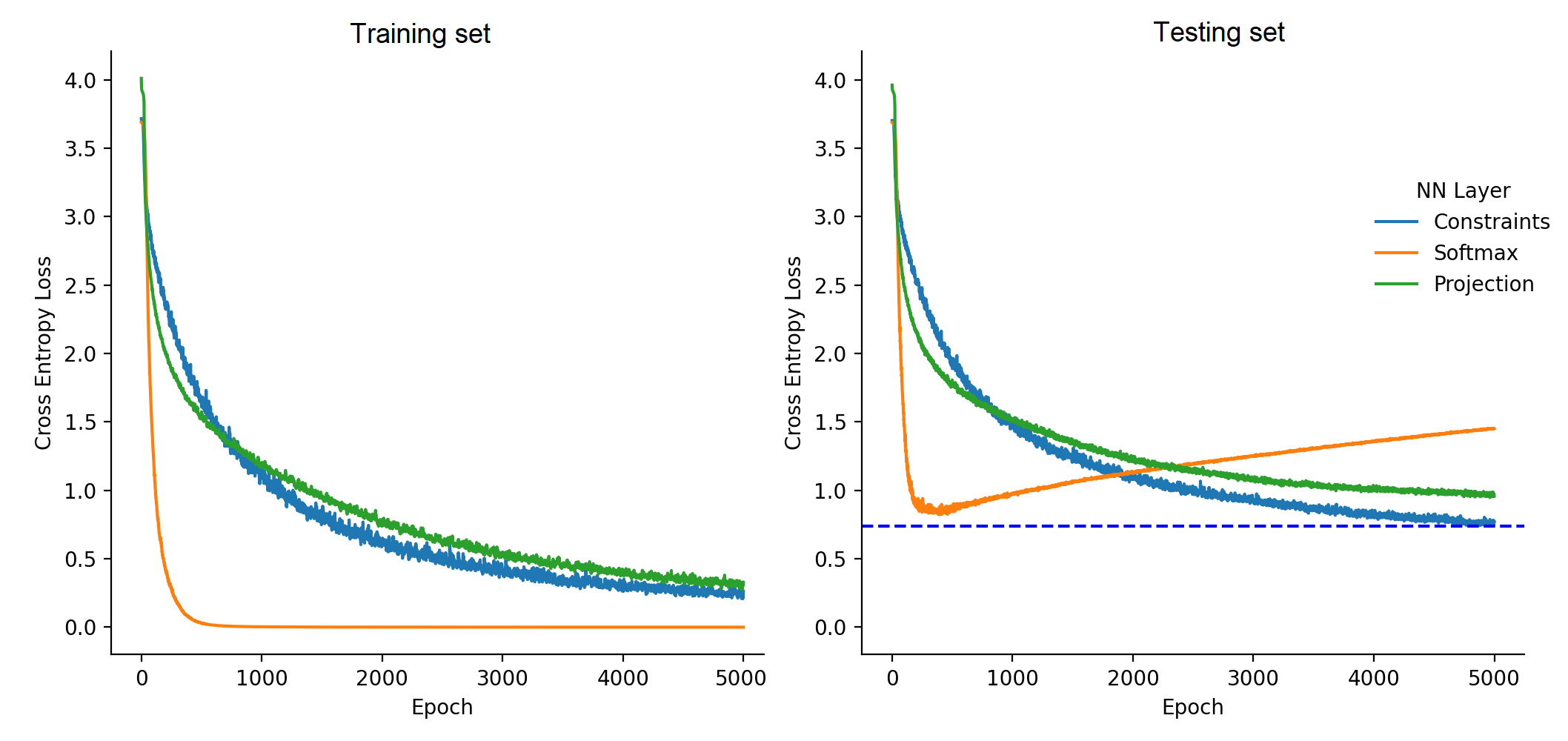}%
\caption{Comparison of cross entropy for different types of the final layer of
the classification neural network}%
\label{fig:classification_comparison}%
\end{center}
\end{figure}
%

\begin{figure}
[ptb]
\begin{center}
\includegraphics[
height=2.4959in,
width=5.6119in
]%
{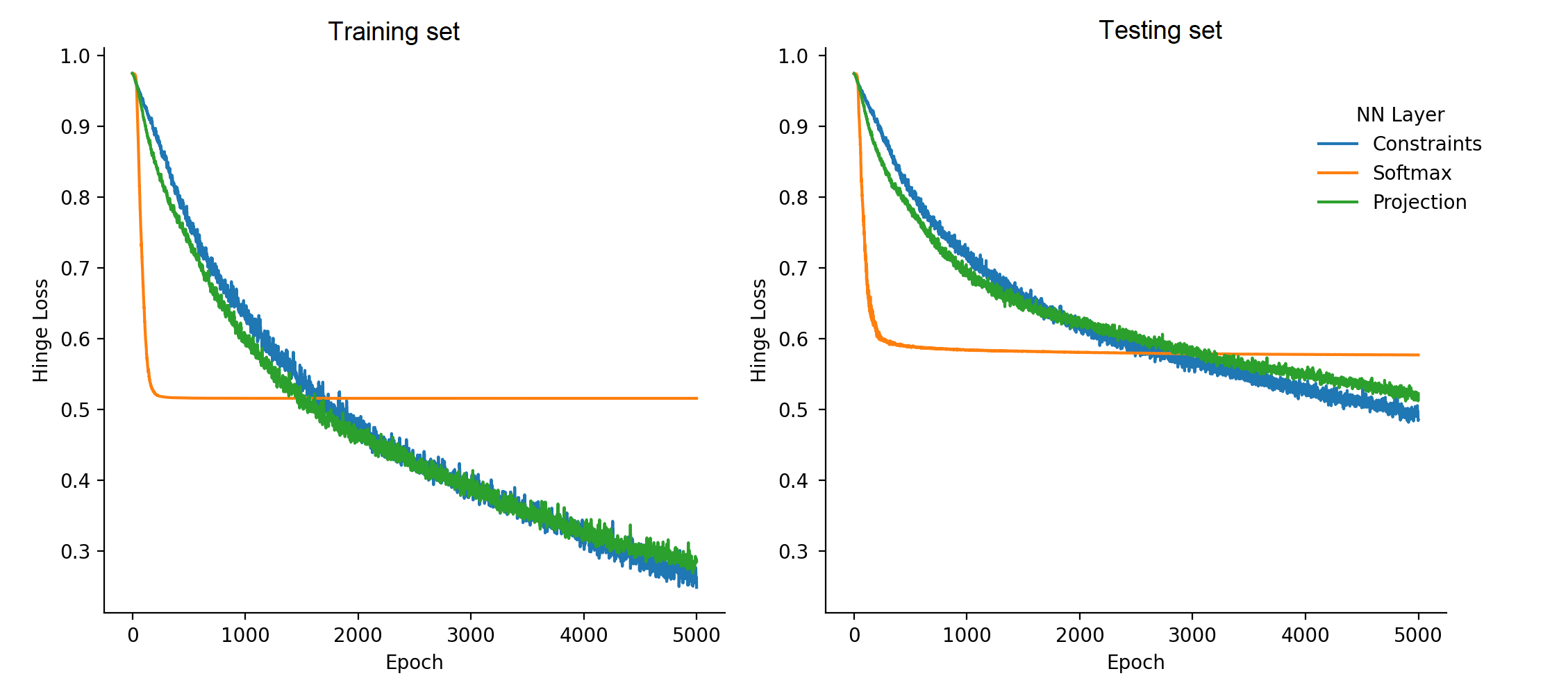}%
\caption{Comparison of \textquotedblleft Hinge loss\textquotedblright \ for
different types of the final layer of the classification neural network}%
\label{fig:classification_comparison_hinge}%
\end{center}
\end{figure}

The dotted line in Fig.\ref{fig:classification_comparison} denotes the minimum
value of the loss function on the testing set. It can be seen from
Fig.\ref{fig:classification_comparison} that all types of layers allow solving
the classification problem. However, the proposed layers in this case provide
slower convergence than the traditional \emph{softmax}. This can be explained
by the logarithm in the cross entropy expression, which is \emph{compensated}
by the exponent in the \emph{softmax}. Nevertheless, the proposed layers can
be regarded as a more general solution. In addition, it can be seen from
Fig.\ref{fig:classification_comparison_hinge} that this hypothesis is
confirmed if another loss function is used, namely \textquotedblleft Hinge
loss\textquotedblright. One can see from
Fig.\ref{fig:classification_comparison_hinge} that the \emph{softmax} also
converges much faster, but to a worse local optimum.

In addition to the standard constraints (\ref{eq:proba_distr}), new
constraints can be added, for example, upper bounds $\overline{p}_{i}$ for
each probability $x_{i}$:
\begin{equation}
f_{\theta}(z)\in \left \{  x~|~0\leq x_{i}\leq \overline{p}_{i},~\mathbf{1}%
^{T}x=1\right \}  .
\end{equation}

This approach can play a balancing role in training, by reducing the influence
of already correctly classified points on the loss function. To illustrate how
the neural network is trained with these constraints and to simplify the
result visualization, a simple classic classification dataset
\textquotedblleft Iris\textquotedblright \ is considered. It contains only $3$
classes and $150$ examples. Three classes allow us to visualize the results by
means of the unit simplex. In this example we set upper bounds $\overline{p}_i = 0.75$ for $i=1,2,3$. Fig.\ref{fig:simplex} shows the three-dimensional
simplices and points (small circles) that are outputs of the neural network
trained on 100, 500, and 1000 epochs. The neural network consists of $5$
layers of size $64$. Colors of small circles indicate the corresponding
classes (Setosa, Versicolour, Virginica). It can be seen from Fig.\ref{fig:simplex} that the constraints affect not only when the output points
are very close to them, but also throughout the network training.%

\begin{figure}
[ptb]
\begin{center}
\includegraphics[
height=2.0157in,
width=5.2284in
]%
{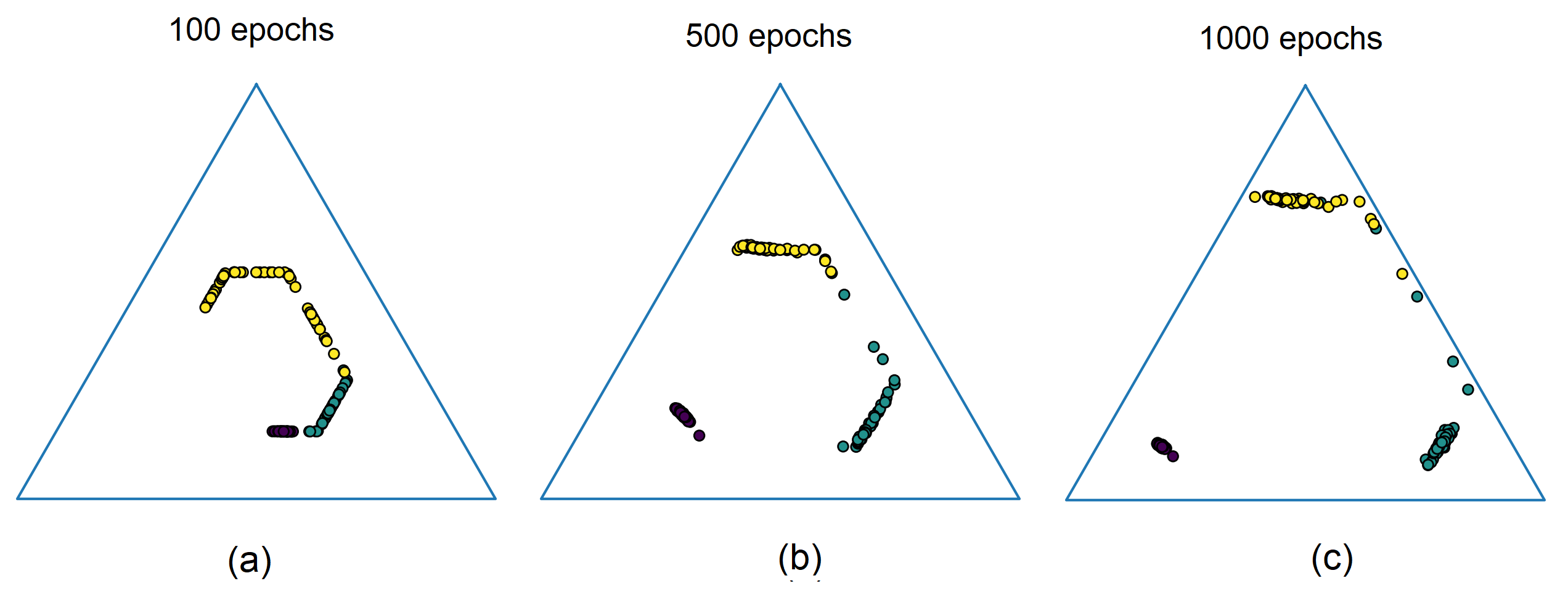}%
\caption{The class probability distributions by different number of epochs for
training: (a) 100 epochs, (b) 500 epochs, (c) 1000 epochs, }%
\label{fig:simplex}%
\end{center}
\end{figure}

\section{Conclusion}

The method, which imposes hard constraints on the neural network output
values, has been presented. The method is extremely simple from the
computational point of view. It is implemented by the additional neural
network layer with constraints for output. Applications of the proposed method
have been demonstrated by numerical experiments with several optimization and
classification problems.

The proposed method can be used in various applications, allowing to impose
linear and quadratic inequality constraints and linear equality constraints on
the neural network outputs, as well as to constrain jointly inputs and
outputs, to approximate orthogonal projections onto a convex set or to
generate output vectors on a boundary set.

We have considered cases of using the proposed method under condition of the
non-convex objective function, for example, the non-convex Bird function used
in numerical experiments set of constraints. At the same time, the feasible
set formed by constraints has been convex because the convexity property has
been used in the proposed method when the point of intersection of the ray $r$
and one of the constraints was determined. However, it is interesting to
extend the ideas behind the proposed method to the case of non-convex
constraints. This extension of the method can be regarded as a direction for
further research.

It should be pointed out that an important disadvantage of the proposed method
is that it works with bounded constraints. This implies that the conic
constraints cannot be analyzed by means of the method because the point of
intersection of the ray $r$ with the conic constraint is in the infinity. We
could restrict a conic constraint by some bound in order to find an
approximate solution. However, it is interesting and important to modify the
proposed method to solve the problems with the conic constraints. This is
another direction for research.

There are not many models that actually realize the hard constraints imposed
on inputs, outputs and hidden parameters of neural networks. Therefore, new
methods which outperform the presented method are of the highest interest.

Another important direction for extending the proposed method is to consider
various machine learning applications, for example, physics-informed neural
networks which can be regarded as a basis for solving complex applied problems
\cite{Kollmannsberger-etal-2021,Raissi-etal-19}. Every application requires to
adapt and modify the proposed method and can be viewed as a separate important
research task for further study.

\bibliographystyle{unsrt}
\bibliography{Constrained_NN}

\end{document}